\documentclass[conference]{IEEEtran}
\IEEEoverridecommandlockouts
\usepackage{cite}
\usepackage{amsmath,amssymb,amsfonts}
\usepackage{graphicx}
\usepackage{textcomp}
\usepackage{xcolor}
\usepackage{booktabs}
\usepackage{comment}
\usepackage[para]{footmisc}
\usepackage{pifont}
\usepackage{float}
\usepackage{makecell}
\usepackage{amsthm}
\usepackage{subfig}

\usepackage{algorithm}
\usepackage{algpseudocode}
\usepackage{amsthm}
\newtheorem{definition}{Definition}
\newtheorem{lemma}{Lemma}
\newtheorem{theorem}{Theorem}

\def\BibTeX{{\rm B\kern-.05em{\sc i\kern-.025em b}\kern-.08em
    T\kern-.1667em\lower.7ex\hbox{E}\kern-.125emX}}

\begin{document}


\title{RIFLES: Resource-effIcient Federated LEarning via Scheduling} 

\author{
\IEEEauthorblockN{Sara Alosaime}
\IEEEauthorblockA{\textit{University of Warwick} \\
Sara.Alosaime@warwick.ac.uk}
\and
\IEEEauthorblockN{Arshad Jhumka}
\IEEEauthorblockA{\textit{University of Leeds} \\
H.A.Jhumka@leeds.ac.uk}
}

\maketitle

\begin{abstract}

Federated Learning (FL) is a privacy-preserving machine learning technique that allows decentralized collaborative model training across a set of distributed clients, by avoiding raw data exchange. A fundamental component of FL is the selection of a subset of clients in each round for model training by a central server. Current selection strategies are myopic in nature in that they are based on past or current interactions, often leading to inefficiency issues such as straggling clients. In this paper, we address this serious shortcoming by proposing the RIFLES approach that builds a novel \emph{availability forecasting} layer to support the client selection process. We make the following contributions: (i) we formalise the sequential selection problem and reduce it to a scheduling problem and show that the problem is NP-complete, (ii) leveraging heartbeat messages from clients, RIFLES build an availability prediction layer to support (long term) selection decisions, (iii) we propose a novel adaptive selection strategy to support efficient learning and resource usage. To circumvent the inherent exponential complexity, we present RIFLES, a heuristic that leverages clients' historical availability data by using a CNN-LSTM time series forecasting model, allowing the server to predict the optimal participation times of clients, thereby enabling informed selection decisions. By comparing against other FL techniques, we show that RIFLES provide significant improvement by between 10\%-50\% on a variety of metrics such as accuracy and test loss. 
To the best of our knowledge, it is the first work to investigate FL as a scheduling problem.

\end{abstract}


\section{Introduction}

With the proliferation of Internet of Things (IoT) devices, users are increasingly collecting  substantial amounts of personal data related to their daily activities. This extensive data supports diverse applications, including human activity recognition~\cite{chen2017performance,khan2022human} and healthcare monitoring~\cite{baines2024patient,allegri2025collecting}. To address privacy concerns inherent in handling such sensitive data, Federated Learning (FL) has emerged as a revolutionary decentralized machine learning paradigm, facilitating collaborative training of models across multiple clients without compromising individual privacy~\cite{mcmahan2017communication}. By conducting computation directly on client devices, FL effectively addresses data governance concerns while preserving privacy~\cite{bonawitz2019konecny}. Initially introduced by Google~\cite{konevcny2016federated}, FL has gained significant traction across various fields, such as healthcare~\cite{grama2020robust} and finance~\cite{long2020federated}. 

In a standard FL scenario, clients collaborate in model training over multiple communication rounds, coordinated by a central server. Typically only subsets of clients participate per round to optimize resource utilization and to ensure model representativeness~\cite{kairouz2021advances}. The selection process can be initiated by either the central server or by the client themselves, both with their own disadvantages. For instance, server-initiated client selection may inadvertently select clients that are unavailable during a particular training round, leading to wasted resources or idle times~\cite{abdelmoniem2021resource,bonawitz2019konecny,konevcny2016federated}. On the other hand, employing client-initiated selection could result in significant communication overhead due to frequent and uncontrolled model updates or introduce bias as faster or more active clients disproportionately influence the global model~\cite{chen2020asynchronous,wang2022asynchronous}. We call these selection strategies \emph{myopic} as they focus on client selection for a given round only, eschewing a more general approach where client participation may be considered across several rounds depending on their respective availabilities.

As such, we study the following problem, called RIFLES: Given the client availability prediction, is it possible to schedule clients to improve FL performance? We first reduce the client scheduling (RIFLES) problem to a resource-constrained task scheduling problem and prove RIFLES to be NP-complete. To address the inherent associated complexity, we propose a novel heuristic, which we called a \emph{greedy heuristic} (GH) for client scheduling. The heuristic, integrated within a server-based middleware (called the RIFLES middleware), leverages the fact that, in a distributed system, nodes often transmit heartbeats to notify of their operational status. RIFLES uses those heartbeats to develop a CNN-LSTM client availability forecasting model to support an optimal client selection strategy for each round.

We implement the RIFLES middleware and integrate two different scheduling algorithms within RIFLES, to produce two variants namely RIFLES-GH and RIFLES-LRU (Least Recently Used). We compare the performances of these RIFLES variants against a number of existing FL techniques. Our results show that RIFLES-GH and RIFLES-LRU significantly improve on previous FL techniques by between 10\% - 50\% on a range of metrics such as accuracy, test loss and client participation among others, showing the efficiency of the customisable RIFLES framework. On the other hand, RIFLES-GH provides better performance than RIFLES-LRU. Table \ref{tab:fl_methods_comparison} compares popular FL methods based on their abilities to address key scheduling criteria, including adaptive round timing, clients capability awareness, clients availability and historical participation tracking.  RIFLES outperforms these approaches by tackling all these criteria, ensuring more efficient FL.

\begin{table}[h!]
\centering
\caption{Comparison of RIFLES and FL baselines Based on Scheduling Criteria.}
\begin{tabular}{c c c c c}
\toprule
\textbf{Method} & \makecell{\textbf{Capability} \\ \textbf{Awareness}} & \makecell{\textbf{Availability} \\ \textbf{Awareness}} & \makecell{\textbf{Historical} \\ \textbf{Participation} \\ \textbf{Awareness}} & \makecell{\textbf{Adaptive} \\ \textbf{Rounds} \\ \textbf{Timing}} \\ 
\midrule
FedAvg~\cite{mcmahan2017communication} & \ding{55} & \ding{55} & \ding{55} & \ding{55} \\ 
FedCS~\cite{nishio2019client} & \ding{51} & \ding{55} & \ding{55} & \ding{55} \\ 
REFL~\cite{abdelmoniem2023refl} & \ding{51} & \ding{51} & \ding{55} & \ding{55} \\ 
RIFLES \\ (Our Method) & \ding{51} & \ding{51} & \ding{51} & \ding{51} \\ 
\bottomrule
\end{tabular}

\label{tab:fl_methods_comparison}
\end{table}

\section{Related Work}\label{sect:relwork}
Despite advancements in FL, many existing approaches still rely on random client selection, resulting in suboptimal model performance and inefficient resource utilisation, particularly in the presence of client heterogeneity and varying availability~\cite{kim2021dynamic, mayhoub2024review}. Over the years, researchers have proposed various client selection strategies to address these challenges, where the server selects a subset of clients (e.g., 10s of clients from 1,000s of clients) to contribute to the global model in each round. Existing client selection strategies vary widely based on system and statistical objectives. Some approaches prioritise clients with superior hardware and network capabilities~\cite{nishio2019client}, while others focus on enhancing statistical efficiency by selecting clients that contribute higher-quality updates~\cite{chen2020optimal,ruan2021towards}. For instance, FedCS\cite{nishio2019client} emphasises system and network performance during client selection, whereas Oort~\cite{lai2021oort} incorporates both system and statistical considerations to optimise the selection process. TiFL~\cite{chai2020tifl} sorts clients into multiple tiers (e.g., fast and slow tiers) for training.

 There exists strategies facing underutilized resources and reduced model coverage, especially under availability heterogeneity~\cite{abdelmoniem2023refl,yang2022flash}. The problem of unpredictable client availability increases the risk of dropouts and stragglers, underscoring the need for availability-aware approaches. To address this challenge, several works have proposed innovative solutions. Particularly, REFL~\cite{abdelmoniem2023refl} and FLASH~\cite{yang2022flash} introduce selection mechanisms that integrate client availability into the selection process. These methods also efficiently integrate stale updates, mitigating the negative consequences of availability heterogeneity in FL. Nonetheless, current client selection strategies encounter are mostly myopic, predicting availability only for the next round while we propose an approach to schedule clients over several rounds.
 

\section{Federated Learning: System Model, Hypothesis and Objectives}\label{sect:models}
We consider a typical cross-device FL system comprising a central server \(PS\) and multiple clients \( \mathcal{N}\).  Clients periodically signal availability (e.g., via heartbeats) and may follow predictable diurnal usage patterns; e.g., individuals who follow daily or weekly routines at home or workplaces have reliable patterns for predicting their availability. Furthermore, we assumed the communication network to be reliable; messages between the server and clients are eventually delivered. We refer to the collective local training of a model on a given client $i$ as a job $J_i$ and one instance of a local training of a model on a client $i$ as a task $T_i^j~(\in J_i)$, where $j$ represents the $j^{th}$ training instance on client $i$.

\noindent\textbf{Objective}:  Mathematically, the objective of FL can be stated as follows: let $p_i$ be the probability of the client $i$ involvement in any given round.   The probability \( p_i \) differs between clients due to variations in their devices availability times and system constraints such as communication efficiency, defined as:
\small
\[
\mathbb{E}[\Delta w] = \eta \sum_{i=1}^N p_i \nabla F_i(w).
\]
\normalsize
where $\mathbb{E}[\Delta w]$ is the expected global model update at round, $\eta$ is the learning rate, $N$ is the total number of clients,  $\nabla F_i(w)$ is the gradient of the local objective function for client $i$ at the current model $w$, $p_i$ is the probability that client $i$ participates in that round. Thus, the aggregated global model update spanning all $T$ rounds,  summation of the anticipated contributions from each client across all rounds, defined as:
\small 

\[
\mathbb{E}[\Delta W] = \eta \sum_{t=1}^T \sum_{i=1}^N p_i \nabla F_i(w) = \eta T \sum_{i=1}^N p_i \nabla F_i(w).
\]
\normalsize
where, $E[\Delta W]$ is the expected global model update over $T$ rounds. This demonstrates higher \( p_i \) biases updates and that scheduling affects FL cost.

\section{Limitations of Existing Client Selection Strategies}\label{sect:background}

In real-world FL settings, full client participation is impractical due to various factors, e.g., unstable connectivity, limited resources or large client pools~\cite{kairouz2021advances, bonawitz2019towards, lai2022fedscale}, making effective client selection essential. Selection typically occurs before each training round, with the server choosing clients based on factors like resource status, device capabilities, data quality or historical performance~\cite{chen2020optimal, wu2020safa, cho2020client}. As results, the client selection mechanism is typically one of two variants: (i) Push mechanism (Server-Initiated mechanism) and (ii) Pull mechanism (Client-Initiated mechanism). 

\begin{figure}
    \centering
   \includegraphics[width=0.8\linewidth]{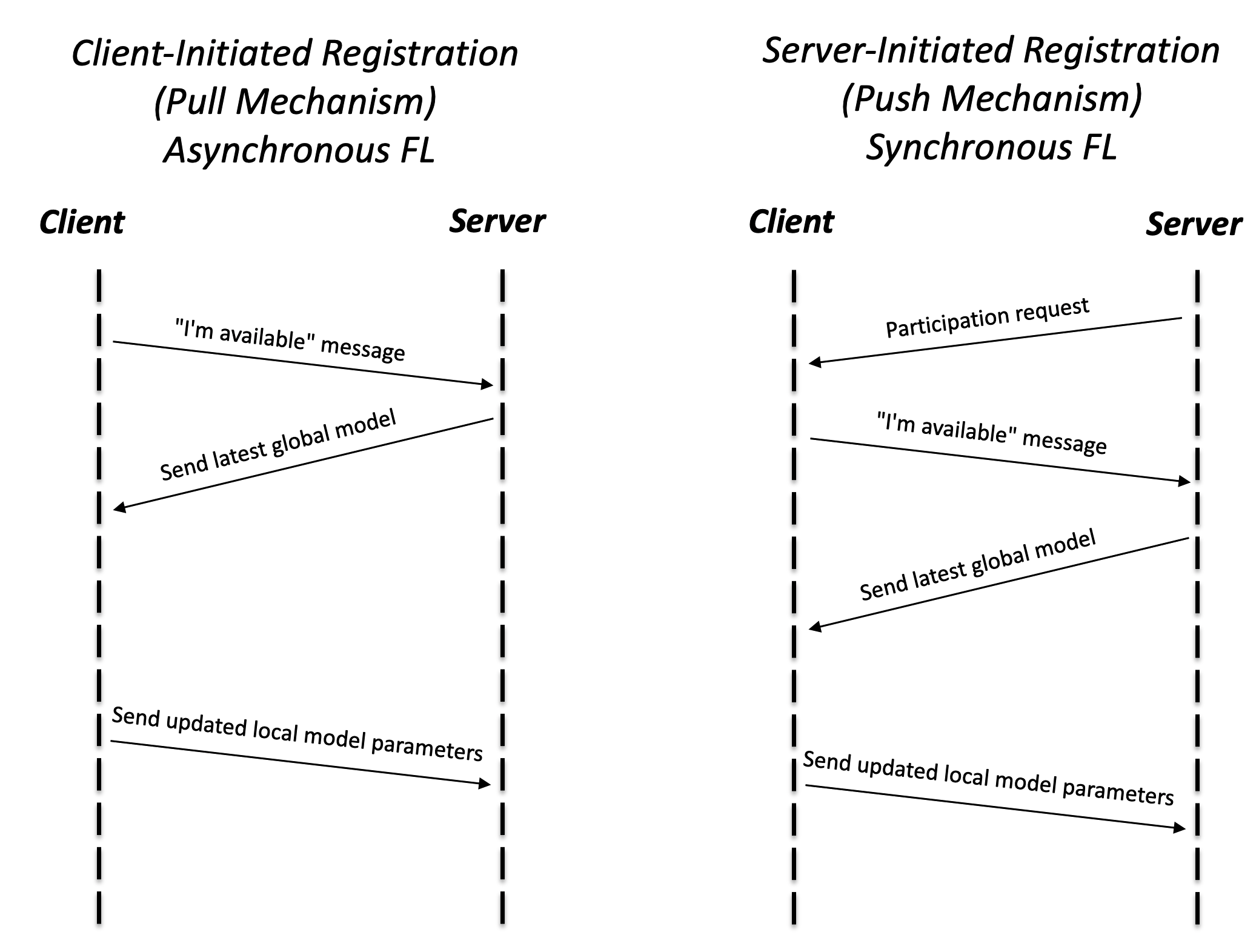}
    \caption{Registration Mechanisms in FL: Client-Initiated (Pull mechanism) vs. Server-Initiated (Push mechanism).}
    \label{fig:enter-label}
\end{figure}
\begin{itemize}
    \item  \textbf{Push Mechanism:} At the start of a round, the server sends out requests to clients to check their availability. Clients convey registration messages, indicating their preparedness to participate.  This mechanism is commonly employed in synchronous FL, where the server collects registration messages and picks a subset of clients for the subsequent  round according to predefined criteria such as resource availability, data diversity or  fairness considerations, e.g., ~\cite{ribero2022federated,abdelmoniem2023refl,yang2022flash}.

    \item  \textbf{Pull Mechanism:} The clients autonomously determine when to join in the training process by proactively sending their status and capabilities to the central server. This approach is frequently utilized in asynchronous FL.
    When the server receives the registration message, it immediately transmits  the latest global mechanism to the client for local updates, e.g.,~\cite{nguyen2022federated,wang2022asynchronous}.

\end{itemize}

Both mechanisms exhibit drawbacks stemming from inadequate scheduling insights and a lack of awareness of client availability across training rounds on the server side, resulting in myopic selection strategies. These may result in sub-optimal performance as depicted in Figure~\ref{fig:fig2} or there is a risk of catastrophic staleness of the slow devices, especially under large-scale
FL~\cite{sun2022fedsea}. Furthermore, if clients train on older versions of the global model before a newer version becomes available, the model's convergence rate and accuracy could decrease~\cite{wu2020safa}. Although the server simultaneously initiates participation requests in the push strategy, the availability of a client can change after the selection phase, resulting in what is termed as ``availability faults''~\cite{alosaime2024flare}.

\begin{figure}
    \centering
   \includegraphics[width=0.9\linewidth, height=0.6\linewidth]{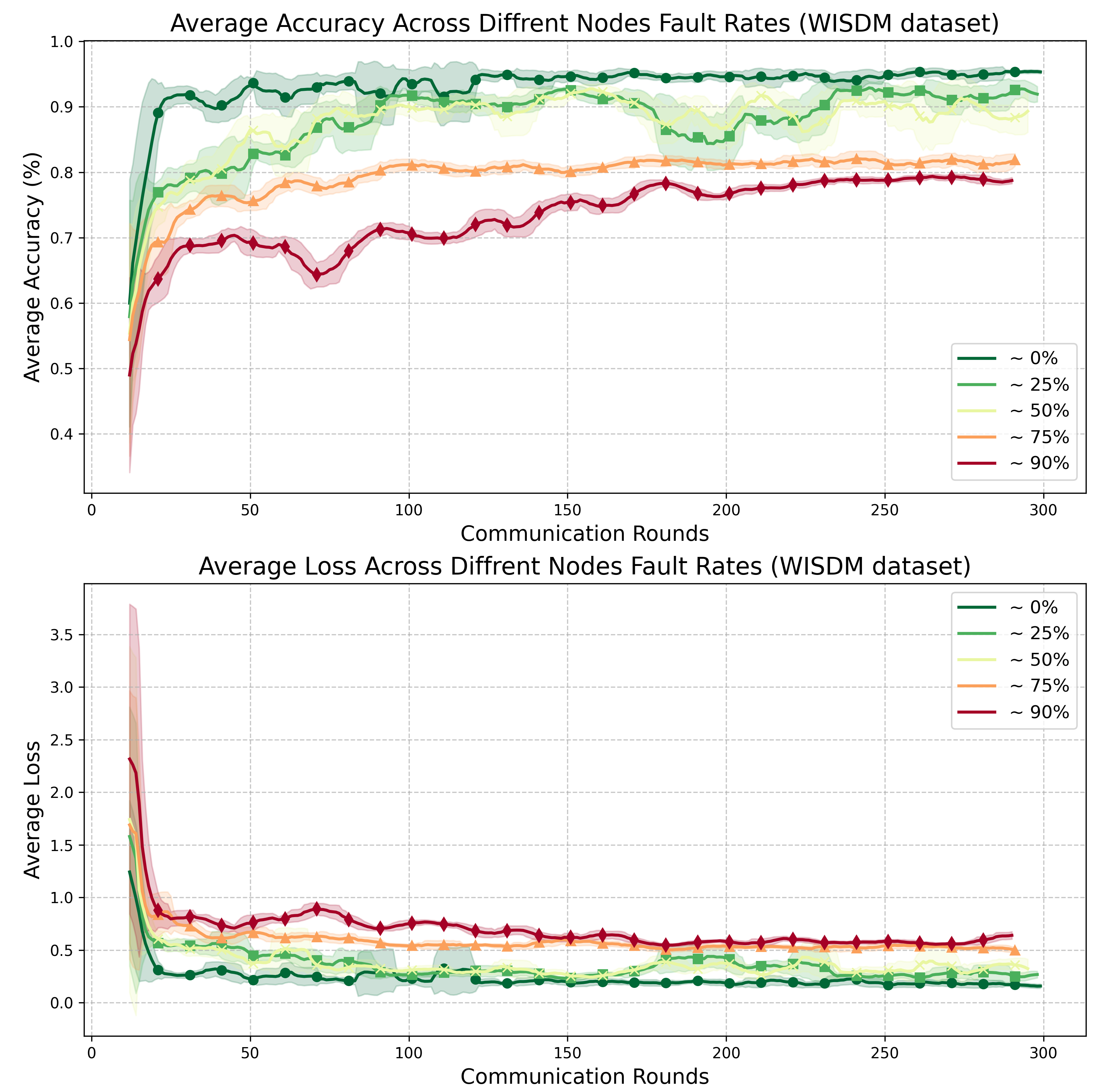}
    \caption{Effect of varying availability fault rates on the performance of REFL~\cite{abdelmoniem2023refl}, a state-of-the-art FL technique.}
    \label{fig:fig2}
\end{figure}

Availability faults mean only a fraction of selected clients participate per round~\cite{abdelmoniem2023refl,abdelmoniem2023comprehensive}, disrupting training coordination, increasing dropout rates, slowing convergence, and degrading global model quality~\cite{wang2019adaptive}, as shown in Figure~\ref{fig:fig2}. Typical mitigation approaches suffer from relying on static assumptions about client availability, which do not adapt to dynamic client behavior.
To address these gaps, we propose the RIFLES approach, inspired by real-time systems, using a time-series forecasting model to enable informed scheduling across training rounds.


\section{Theory}\label{sect:theory}
In this section, we first formalise the RIFLES problem and subsequently study the complexity of the client scheduling problem. In the formalisation below, the parameter $\beta$, termed as the local job execution proportion, captures the availability proportion of a client for training across rounds. Setting $\beta = 100\%$ means that all clients needs to participate in every round. On the other hand, the parameter $\alpha$ captures the proportion of clients to participate in training in any given round, i.e., $\alpha = 100\%$ means that all available clients need to participate in training in every round.

\begin{definition}[RIFLE Scheduling (RIFLES)]
   Given a number $n \in Z^+$ of clients $\{1,\dots, n\}$, a given set $J = \{J_1, \ldots, J_n\}$ of training jobs, where each job $J_i \in J$ is a set of training tasks $T_i^1,\ldots, T_i^K$, with each task $t$ having length $l(t) = 1$, a global job selection proportion $\alpha$, a local job execution proportion $\beta$ and a training deadline $p \in Z^+$, does there exist an $n$-client schedule $\phi$ for $J$ that (i) meets the overall deadline $p$, (ii) no more than ($\alpha\cdot n$) clients are executing tasks at any point in time (i.e., in any time slot), (iii) at least ($\beta \cdot K$) tasks are executed for each job $J_i$ and (iv) all tasks $T_i^j$ of job $J_i$ execute on the same client $i$?  
\end{definition}

\begin{lemma}[RIFLES and class of NP]
\label{lem:np}
    RIFLES is in NP.
\end{lemma}

\begin{proof}
    To prove this, we need to verify the correctness of a given possible solution $R^*$ of RIFLES in polynomial time. 

    To check conditions (i) - (v) requires checking all $p$ periods across all $n$ clients, making the verification process $O(pn)$.
\end{proof}

We will now prove that RIFLES is NP-complete by showing a reduction to the problem of resource-constrained scheduling~\cite{garey1979npc}, which we define formally now.

\begin{definition}[Resource-Constrained Scheduling (RCS)]
   Given a set $T$ of tasks, with each task $t$ having length $l(t) = 1$, a number $m \in Z^+$ of processors, a number $r \in Z^+$ of resources, resource bounds $B_i, 1 \leq i \leq r$, resource requirement $R_i(t), 0 \leq R_i(t) \leq B_i$, for each task $t$ and an overall deadline $D \in Z^+$, does there exist an $m$-processor schedule $\sigma$ for $T$ that meets the overall deadline $D$ and obeys the resource constraints, i.e., such that $\forall u \geq 0$, if $S(u)$ is the set of all tasks $t \in T$ for which $\sigma(t) \leq u < \sigma(t)~+~l(t)$, then for each resource $i$, the sum of $R_i(t)$ over all $t\in S(u)$ is at most $B_i$?
\end{definition}

\begin{lemma}[RIFLES and NP-hardness]
\label{lem:nph}
    RIFLES is NP-hard.
\end{lemma}

\begin{proof}
    To prove that RIFLES is NP-hard, we first show a mapping between PCS and RIFLES and then reduce PCS to the RIFLES problem.

    \noindent \textbf{Mapping:}
\small
    \begin{itemize}
        \item $\phi \mapsto \sigma$
        \item $\bigcup_{i=1}^n J_i \mapsto T$
        \item $\alpha \mapsto 100\%$
        \item $\beta \mapsto 100\%$
        \item $p \mapsto D$
        \item $1 \mapsto r$
        \item $\alpha n \mapsto B_1$
        \item $1 p \mapsto R_1(t)$
    \end{itemize}
\normalsize
We now need to show that a solution for RIFLES exists if and only if a solution for RCS exists.

\begin{itemize}
    \item[] 
    \item[$\Leftarrow$] We show how a solution for RIFLES, i.e., $\phi$, can be obtained from a solution for RCS, i.e., $\sigma$. Because $\sigma$ solves RCS, under the identified mapping, $\sigma$ satisfies conditions (i) - (iii) of RIFLES. However, condition (iv) may not be satisfied and has to be resolved, as follows: Condition (iv) is violated when a task $T_i^j$ is executing on a processor $k (\not = i)$ in any given slot $\tau$. Thus, to enforce condition (iv), for every slot $\tau$ in $\sigma$, any task $T_i^j$ (of job $i$) executing on a processor $k (\not = i)$ in any slot $\tau$ needs be swapped with any task $T_m^n$ executing on processor $i$ in $\tau$ (or moved if processor $i$ is idle in $\tau)$. This procedure is executed repeatedly until all tasks in a given slot are executing on their correct processors, resulting in $\phi$.
    \item[]
    \item[$\Rightarrow$] A solution $\phi$ of RIFLES trivially satisfies the requirements for RCS, resulting in $\sigma$.

\end{itemize}
\end{proof}

\begin{theorem}[RIFLES and NP-completeness]
    RIFLES is NP-complete.
\end{theorem}

\begin{proof}
    The proof follows from Lemmas~\ref{lem:np} and~\ref{lem:nph}.
\end{proof}


\section{System Design}\label{sect:systemdesign}

Given the complexity of client scheduling in FL, we introduce RIFLES, a customizable FL stack designed to support various scheduling strategies. RIFLES incorporates an availability forecasting layer on the server, leveraging periodic heartbeat messages from clients to monitor and predict their availability, as in distributed systems, supporting smarter long-term selection and scheduling decisions.  Figure~\ref{fig:RIFLES} illustrates the architecture and core components of RIFLES, which we detail step-by-step from A to E.
\begin{figure}
    \centering
    \includegraphics[width=1\linewidth]{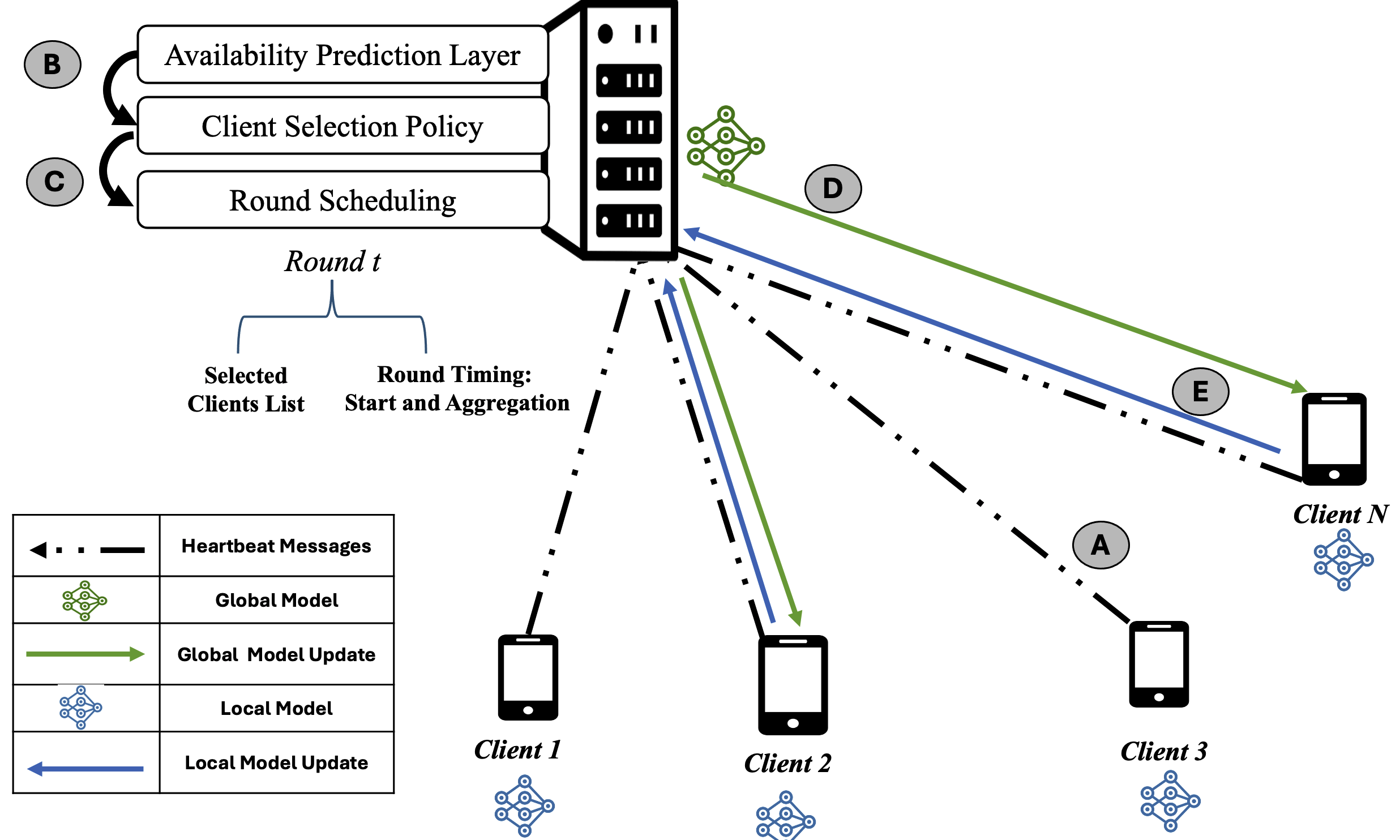}
    \caption{Overview of the RIFLES Framework.}
    \label{fig:RIFLES}
\end{figure}


\subsection{Availability Prediction Using Heartbeat Updates}

\subsubsection{\textbf{Availability Status Updates- Heartbeats Mechanism}}

In RIFLES, we propose \emph{novel} availability awareness mechanism by leveraging the fact that clients often send heartbeats to notify of their operational statuses. As such, we assume that, once clients decide to participate in FL training, they will periodically send their presence to the server in terms of periodic timestamped heartbeats. We use a client’s \(hb_{i}^{\text{t}} \) (heartbeat) as a proxy for the client \(i\)'s current availability status at given time \(t\), carrying its availability information as payload. The server sets \(hb_{i}^{\text{t}} = 1\), for client \(i\) at time \(t\) if the client reports being connected to WiFi, sufficiently charged and in an idle state\footnotetext{Idle status refers to a state where the device is not actively being used by the user to run other applications, allowing it to perform computations without disrupting their experience.}, else \(hb_{i}^{\text{t}} = 0\).  Due to issues, e.g., network conditions and client mobility, we assume that most heartbeats are transmitted reliably and only a small proportion \(\epsilon \) may be delayed or lost, capturing real-world conditions as RIFLES is intended to operate in a WiFi setting, :

\[
\frac{\text{Heartbeats not received}}{\text{Total heartbeats expected}} \leq \epsilon
\]
\noindent where \( \epsilon \) is a predefined threshold ensuring that heartbeat loss remains within acceptable limits to maintain prediction stability. To address lost heartbeats, timestamps are used to map them to the correct time slots in the daily matrix, ensuring accurate availability updates.

\subsubsection{\textbf{Daily Availability Matrix Structure}}
To comprehensively capture client availability over time, we represent it as a series of daily matrices $\mathcal{M}$, defined as:
\noindent \[
\mathcal{M} = \left\{ \mathbf{M}^{(1)}, \mathbf{M}^{(2)}, \dots, \mathbf{M}^{(d)} \right\}
\]  

where \(\mathcal{M}\) is the set of daily matrices, with each \( \mathbf{M}^{(d)} \) representing a specific day \( d \) in a total of \( D \) tracking days.
Each matrix \( \mathbf{M}^d \) is of dimensions \( S \times N \), where \( S \) is the total number of discrete time slots in a day\( d \) and \( N \), the number of clients being tracked. Each time slot in \( S \) represents a fixed interval \( \Delta t \) (e.g., 2 minutes), resulting in: \[
S = \frac{\text{Day duration (in minutes)}}{\Delta t} = \frac{1440}{\Delta t}  =\text{  720  slots per day.}
\]

The entry \( A_i^{d}(s) \) in the matrix \( M^d \) represents the availability status of client \( i \) at time slot \( s \) in day \( d \), where:

\[
A_i^{d}(s) = 
\begin{cases}
1, & \text{if client } i \text{ is "available" in slot } s \text{ on day } d. \\
0, & \text{if client } i \text{ is "unavailable" in slot } s \text{ on day } d.
\end{cases}
\]

 For example, the availability matrix \( \mathbf{M}^{(d)} \):
\small
\begin{center}
\scriptsize
\begin{tabular}{|c|c|c|c|c|c|}
\hline
\textbf{S} & \textbf{Client 1} & \textbf{Client 2} & \textbf{Client 3} & $\dots$ & \textbf{Client N} \\
\hline
1 & \( A_1^{1}(1) \) & \( A_2^{1}(1) \) & \( A_3^{1}(1) \) & $\dots$ & \( A_N^{1}(1) \) \\
2 & \( A_1^{1}(2) \) & \( A_2^{1}(2) \) & \( A_3^{1}(2) \) & $\dots$ & \( A_N^{1}(2) \) \\
3 & \( A_1^{1}(3) \) & \( A_2^{1}(3) \) & \( A_3^{1}(3) \) & $\dots$ & \( A_N^{1}(3) \) \\
$\vdots$ & $\vdots$ & $\vdots$ & $\vdots$ & $\ddots$ & $\vdots$ \\
S & \( A_1^{1}(S) \) & \( A_2^{1}(S) \) & \( A_3^{1}(S) \) & $\dots$ & \( A_N^{1}(S) \) \\
\hline
\end{tabular}
\normalsize
\end{center}
\normalsize
The value of \( A_i^{d}(s) \) is derived from the heartbeat signals received from client \(i\), as detailed next.

\subsubsection{\textbf{{Heartbeat Signal Processing}}}

To determine the time slot \( s \) corresponding to a heartbeat timestamp \( hb^{\text{t}} \), we use the formula \( s = \left\lceil \frac{t}{\Delta t} \right\rceil \); where \( t \) is the timestamp and \( \Delta t \) is the duration of each time slot.  The heartbeat belongs to the 60th time slot (\( s = 60 \)). Once a heartbeat is received, its status is considered valid for the next \( W_i \) time slots, where \( W_i \) is a client-specific validity window. Thus, the availability status \( A_i^{(d)}(s) \) is updated as follows:
\small
\begin{itemize}
    \item  If \( \text{hb}_i^{t} = 1 \) (client is available):
  
  \[
  A_i^{(d)}(s) = 1, \, \forall s \in [s, s + W_i]
  \]

 \item  If \( \text{hb}_i^{t} = 0 \) (client is (un)available):
  
  \[
  A_i^{(d)}(s) = 0, \, \forall s \in [s, s + W_i]
  \]

  \end{itemize} 
\normalsize
If a new heartbeat with the opposite status is received within the validity window, the availability status is updated beginning with the new heartbeat's time slot onwards. In general, we  can formalize the availability status \( A_i^{(d)}(s) \) as:
\small
\[
A_i^{(d)}(s) = 
\begin{cases}
\text{status of } hb_i^t, & \text{if } hb \text{ exists for all } s \in [s, s + W_i] \\
0, & \text{otherwise}
\end{cases}
\]
\normalsize

\subsection{Availability Prediction Layer}
 \subsubsection{\textbf{CNN-LSTM Model as a Time Series Prediction Model}}
The server processes the heartbeat data as daily matrices \(\mathcal{M}\) through the Availability Prediction Layer, utilizing a CNN-LSTM model to identify clients' availability for the next day. The CNN-LSTM model architecture consists of a combination of (CNN) and (LSTM) networks. The CNN is used to capture spatial patterns in client availability across time slots within a day, while the LSTM captures temporal dependencies across multiple days.  
\begin{itemize}
    \item Spatial Feature Extraction with CNN:
     For each day \( d \), the CNN processes \( \mathbf{M}^{(d)} \) to extract spatial features \( \mathbf{F}^{(d)} \). This  captures local patterns of client availability within the day, such as peak availability times. 
    \item Temporal Dependency Modeling with LSTM: 
    After that, the sequence of feature vectors is  utilise as input into the LSTM network to model temporal dependencies across days.    The sequence \( \{ \mathbf{f}^{(1)}, \mathbf{f}^{(2)}, \dots, \mathbf{f}^{(D)} \} \) is input into the LSTM. 
\end{itemize}

The model takes availability matrices from the past \( d \) days as input, leveraging temporal information to predict future availability \( PA \) for clients for the next day.  

\begin{center}
\scriptsize
\begin{tabular}{|c|c|c|c|c|c|}
\hline
\textbf{S} & \textbf{Client 1} & \textbf{Client 2} & \textbf{Client 3} & $\dots$ & \textbf{Client N} \\
\hline
1 & \( PA_1^{1}(1) \) & \( PA_2^{1}(1) \) & \( PA_3^{1}(1) \) & $\dots$ & \( PA_N^{1}(1) \) \\
2 & \( PA_1^{1}(2) \) & \( PA_2^{1}(2) \) & \( PA_3^{1}(2) \) & $\dots$ & \( PA_N^{1}(2) \) \\
3 & \( PA_1^{1}(3) \) & \( PA_2^{1}(3) \) & \( PA_3^{1}(3) \) & $\dots$ & \( PA_N^{1}(3) \) \\
$\vdots$ & $\vdots$ & $\vdots$ & $\vdots$ & $\ddots$ & $\vdots$ \\
S & \( PA_1^{1}(S) \) & \( PA_2^{1}(S) \) & \( PA_3^{1}(S) \) & $\dots$ & \( PA_N^{1}(S) \) \\
\hline
\end{tabular}
\normalsize
\end{center}

\subsubsection{\textbf{Responded Duration}}
A client's eligibility for a training round depends on its response time and availability duration, ensuring compatibility with device heterogeneity and varying capabilities. Let \( C_{\text{expected}}(i) \) denote the expected duration for client \( i \) to complete local training and return its update, estimated by averaging its response times from past participation rounds, as follows:
\small
\[
C_{\text{expected}}(i) =
\begin{cases}
\frac{1}{n} \sum_{j \in S_i} C_{\text{response}}(i, j), & \text{if } n > 0 \\
C_{\text{init}}, & \text{if } n = 0
\end{cases}
\]
\normalsize

where, $C_{\text{expected}}(i)$ is the expected responded (computation and communication ) duration for client $i$; $n$ is the number of rounds the client participated in. $C_{\text{response}}(i, j)$ is the response time in round $j$ and $C_{\text{init}}$ is the initial response time estimate when $n = 0$.  Thus, continuously tracking and updating clients' training duration is vital for constructing the eligibility matrix, as it relies on both the prediction matrix and the clients' training duration.

\subsubsection{\textbf{Eligibility Matrix Construction}}
The methodology for evaluating client eligibility at a given slot \( s \) for executing a training task is detailed in this section. This process is critical for identifying optimal slots throughout the day for training and selecting the most eligible set of clients for participation, as illustrated in Figure~\ref{fig:Pipeline}. This  eligibility matrix plays a pivotal role in enabling the server to efficiently schedule clients for the following day, ensuring effective client selection and load balancing in federated training over time.

\begin{figure}[H]
    \centering
    \includegraphics[width=1\linewidth]{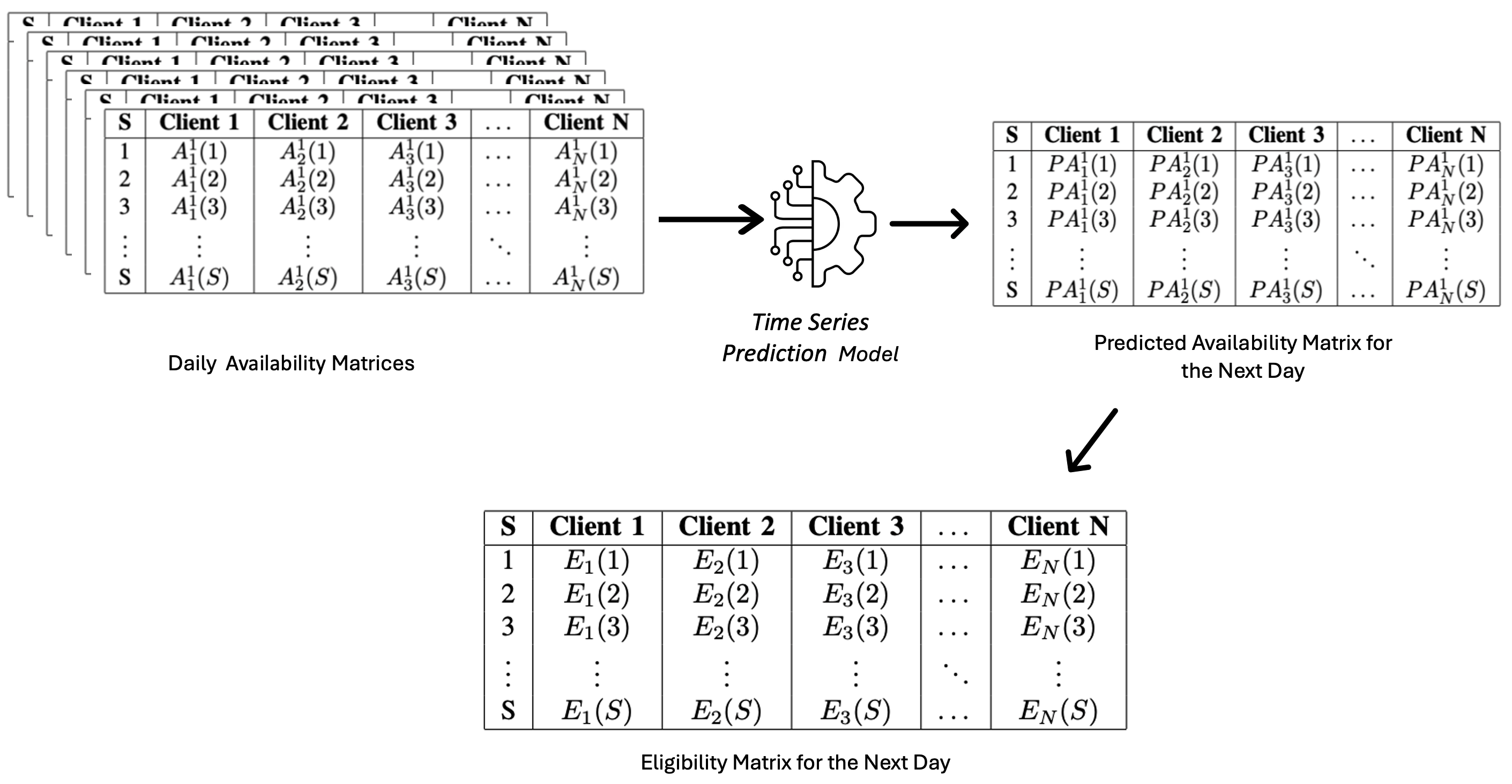}
    \caption{Pipeline for Generating the Eligibility Matrix.}
    \label{fig:Pipeline}
\end{figure}

Considering the prediction matrix and response duration, the following steps are performed for each client \( i \in N \) to generate the eligibility matrix for the next day:
\begin{itemize}
    \item For \( s = 0 \) to \( S \), we calculate the number of consecutive slots from \( s \) onward during which the client is predicted to remain available \(PA =1\). This predicted availability window is denoted as \( \Lambda_i^s \) for client \( i \), representing the number of slots during which the client is expected to stay available after slot \( s \).
    
    \item To ensure that client \( i \) can perform the training task, \( \Lambda_i^s \), the predicted availability window starting from \( s \), must be sufficient. This depends on the expected duration required for the task \( C_{\text{expected}}(i) \), which is specific to client \( i \), plus a constant buffer \( k \) to account for potential communication delays and network variability. The client \( i \) is considered eligible at slot \( s \) if the following condition holds:

\[
 \Lambda_i^s  \geq C_{\text{expected}}(i) + k \Rightarrow E_i(s) = 1
\]

If this condition is met, the eligibility indicator \( E_i(s) \) is set to 1, indicating that client \( i \) is available and may be selected for training at slot \( s \) and \(E_i(s) = 1\). In contrast, if the condition is not met (i.e., if \( \Lambda_i^s \) is too short and not covered \( C_{\text{expected}}(i) \)), the eligibility indicator \( E_i(s) \) is set to 0, indicating that client \( i \) is not available for training at slot \( s \):

\[
\Lambda_i^s  < C_{\text{expected}}(i) + k \Rightarrow E_i(s) = 0
\]
\end{itemize}

Be doing that, we convert the predicted availability matrix into a more informative matrix, the Eligibility Matrix.

\begin{center}
\scriptsize
\begin{tabular}{|c|c|c|c|c|c|}
\hline
\textbf{S} & \textbf{Client 1} & \textbf{Client 2} & \textbf{Client 3} & $\dots$ & \textbf{Client N} \\
\hline
1 & \( E_1(1) \) & \( E_2(1) \) & \( E_3(1) \) & $\dots$ & \( E_N(1) \) \\
2 & \( E_1(2) \) & \( E_2(2) \) & \( E_3(2) \) & $\dots$ & \( E_N(2) \) \\
3 & \( E_1(3) \) & \( E_2(3) \) & \( E_3(3) \) & $\dots$ & \( E_N(3) \) \\
$\vdots$ & $\vdots$ & $\vdots$ & $\vdots$ & $\ddots$ & $\vdots$ \\
S & \( E_1(S) \) & \( E_2(S) \) & \( E_3(S) \) & $\dots$ & \( E_N(S) \) \\
\hline
\end{tabular}
\normalsize
\end{center}

\subsection{Adaptive Selection and Scheduling Policy}

In RIFLES, selection and scheduling are critical processes that involve deciding which clients will participate in contributing their local updates to the global model based on criteria like availability, capabilities, historical participation or fairness when to optimally execute these training rounds. 

In this section, we present two adaptive selection and scheduling policies to showcase the customisable nature of RIFLES: \textbf{Greedy Heuristics (GH)} and  \textbf{Least Recently Used (LRU)}. Both policies utilize the eligibility matrix
that aims to enhance participation rates, accelerate convergence minimise dropout rates, thereby improving the overall efficiency of FL system

\subsubsection{\textbf{Greedy Heuristic (GH) Policy for Client Selection and Slot Scheduling}}
The Greedy Heuristic (GH) policy is designed to select clients and allocate time slots for a predetermined number of training rounds \(R\) (e.g., 24 rounds per day). Its goal is to maximize client participation, maintain a minimum gap \(G\) between consecutive rounds (e.g., 2 slots) and prioritize the inclusion of as many unique clients as possible to promote data diversity. This approach prioritizes availability while striving to optimize client diversity and participation rates within the scheduling constraints. Therefore, as inputs we had (i) Eligibility matrix, the binary matrix indicating whether client \( i \) is eligible at slot \( s \). (ii) Minimum gap \( G \),  The minimum number of slots required between two selected training slots.
(iii) Threshold for participation number\( K_{\text{min}} \), which indicates to the minimum number of clients required  to perform  a training round. Its works as follows: Firstly, calculates the number of eligible clients for each time slot \( s \) by summing the eligible clients across all slots. This step helps prioritize slots with the highest potential participation for scheduling training rounds.

\[
    \text{Eligible Clients}(s) = \sum_{i=1}^N E(i, s)
\]



   After that, we sort slots by count in descending order, based on the count of Eligible Clients. We then select slots with a gap constraint, by initializing an empty list for selected slots and we then iterate over the sorted list of slots. For each slot \( s \), we check if it’s at least \( G \) slots away from all previously selected slots.  If it satisfies this condition and has enough eligible clients\((\text{Eligible Clients}(s) \geq K_{\text{min}})\), we add that slot to the selected slots list. This is continued until a sufficient number of slots is selected or no further slot meets the criteria. After selecting the optimal slots, we calculate the total number of distinct clients covered across these slots. If not all of the clients are included, we adjust the threshold or gap between rounds to maximize the participation of as many unique clients as possible. The unique clients are defined as those with a low eligibility rate, meaning that those whose number of eligible slots falls below a specified threshold \(\alpha\).
    
\[
U = \{ i \in N \mid |\text{EligibleSlots}_i| < \alpha \}
\]
where \( U \) represents the set of unique clients, \( N \) is the total set of clients, \( \text{EligibleSlots}_i \) is the set of eligible slots for client \( i \) and \( \alpha \) is the predefined threshold for eligibility.  These clients are often overlooked in existing client selection approaches due to their sporadic presence. By including them, the policy promotes diversity in client selection, enabling the global model to benefit from underrepresented data distributions and preventing over-reliance on frequently available clients. Finally, compute aggregation time for each round  \( s \in \mathcal{S} \) as $
\text{Agg}_s = \max_{i \in s} C_{\text{expected}}(i), \quad \forall s \in \mathcal{S}$.

\normalsize

\subsubsection{\textbf{Least Recently Used (LRU) for Clients Selection}}

It maintains a dynamic cache to track the order of client participation, ensuring that clients with longer periods of inactivity are given higher priority for selection. Its works as follows: Establish an LRU cache \(\mathcal{Q}\) as a deque with a maximum length of the entire number of clients. Initially, all clients are added to the cache in with the least recently used clients at the beginning. After that, Sort slots by eligible clients in descending order and select slots with gap constraint  For each selected slot \(s\), determine the eligible clients from eligibility matrix \(E\): 
    \[E_s = \{ i \mid E(s,i) = 1 \}\]
   
Then,  filter \( E_s \) based on the LRU order by selecting only those clients who are in \( \mathcal{Q} \). This subset of least recently used eligible clients is denoted as \( L_s \):
\[
L_s = \{ i \in \mathcal{Q} \mid i \in E_s \}
\]
Select up to $K_{\text{min}}$ clients from \( L_s \), prioritizing those at the front of the LRU cache \( \mathcal{Q}\), as they are the least recently used:
\[
\Lambda_s = L_s[:K_{\text{min}}]
\]
 where, $K_{\text{min}}$ is the required number of clients to participate in the current round. Compute the aggregation time for each round \( s \in \mathcal{S} \) as $
\text{Agg}_s = \max_{i \in s} C_{\text{expected}}(i), \quad \forall s \in \mathcal{S}$.

\subsection{Model Distribution and Aggregation}

The server distributes the latest global model to selected clients, who perform local training on their private data.  After training, clients send their updates back to the server, where they are aggregated to form an updated global model.

\section{Experiment Setup}\label{sect:exptsetup}
\subsection{System Implementation}
We conducted our experiments utilizing the FedScale framework~\cite{lai2022fedscale, abdelmoniem2021resource} and leveraged an NVIDIA GPU cluster, approx. 22GB GPU RAM, CUDA compute capability 7.5, to simulate client training processes using PyTorch.  Similar to REFL \cite{abdelmoniem2021resource}, we employed the YoGi optimizer \cite{reddi2020adaptive} for model aggregation. For the system training configuration, we set the number of clients to \(N = 100\) with a client participation rate of \(0.1\) across \(1k\) communication rounds. Regarding the model-specific training setups, for the WISDM dataset, we trained a CNN using a batch size \(B\) of 32, a learning rate \(\gamma=0.005\) and a local training frequency of \(E = 1\). For the CIFAR-10 dataset, we trained a ResNet-18 model with a \(B=32\), a  \(\gamma=0.05\), and \(E = 1\).
\subsubsection{Simulating System Heterogeneity}

 To simulate computational heterogeneity among clients, we adopted a regression-based approach inspired by prior work \cite{yang2022flash, zhang2021nn}. We collected hardware performance data from three representative devices to train a regression model to estimates training speeds. For communication heterogeneity, we used a dataset from \cite{yang2022flash} containing down/upstream bandwidth measurements between 30 devices and the server.

 \subsubsection{Simulating Availability Heterogeneity}
 
 To generate realistic availability patterns for clients over a week, we generated day 1 availability, as reference day, by create random availability patterns for each client, with an increased availability factor of 1.5× during nighttime hours (10 PM to 6 AM) to reflect typical user behavior in real-world, where most individuals tend to charge their devices at night and not used them. Then, for subsequent days, we introduced a 20\% chance for each client to change their availability status hourly, simulating natural fluctuations in user behavior. Moreover, we  introduced random short-term unavailability periods (e.g., clients going offline for 10 minutes after being online for 30 seconds) to mimic brief connectivity losses. Finally, each client was randomly assigned a state trace and hardware capacity to mimic real-world variability.

\subsection{Application, Datasets and Models}

We evaluate the performance across two application domains: human activity recognition (HAR) and image classification. For HAR, we employ the WISDM dataset~\cite{weiss2019wisdm}, where each client receives samples from all 6 activity classes with varying class distributions, simulating realistic and diverse usage conditions. For image recognition, we utilize the CIFAR-10 dataset~\cite{krizhevsky2014cifar10}, introducing more severe heterogeneity by applying a filter class ratio of 0.7 of 10 classes to further challenge the learning process under non-IID settings.

\subsection{Baselines} To evaluate the performance of RIFLES, we compare RIFLES against three baselines:

\begin{itemize}
    \item Random~\cite{mcmahan2017communication}: This selection approach utilised by FedAvg, the most classical synchronous approach, where $\lceil \beta \times N \rceil$ clients randomly selected during the selection phase.

   \item  FedCS~\cite{nishio2019client}: A framework that addresses  heterogeneity by selecting $\lceil \beta \times N \rceil$ clients, collecting resource information and estimating speeds to select those capable of downloading, training and uploading within the deadline.

   \item REFL~\cite{abdelmoniem2023refl}: A state-of-the-art FL framework that addresses resource and availability heterogeneity by evaluating device performance and leveraging client-reported availability for selection. It selects  \(\lceil \beta \times N \rceil\) clients, prioritizing those likely to be unavailable in upcoming rounds.
       
\end{itemize}

\subsection{Evaluation Metrics}
To evaluate the overall performance of RIFLES vs baseline approaches, we utilize the evaluation metrics include Accuracy@Deadline, measuring global model accuracy after 1k rounds; Round of Arrival (RoA@x), indicating how quickly the model reaches a specified accuracy; Completion Rate, the percentage of clients successfully submitting updates by deadline; Successful Rate, the percentage of clients completing training tasks within each round; Dropout Rate, the percentage of selected clients failing tasks and Unique Client Participation Count, representing clients newly participating after a defined number of previous rounds (e.g., 3 rounds).

\section{Results}\label{sect:results}

\begin{table}
\centering
\scriptsize
\caption{Round of Arrival (RoA) at different accuracy thresholds (50\%, 75\% and 90\%) and the final accuracy.}
\setlength{\tabcolsep}{2pt} 
\renewcommand{\arraystretch}{1.1} 
\begin{tabular}{ccccc}
\hline
\textbf{Method} & \multicolumn{4}{c}{\textbf{WISDM Dataset}} \\ \cline{2-5} 
                & \textbf{RoA@50\%} & \textbf{RoA@75\%} & \textbf{RoA@90\%} & \textbf{Acc@Deadline} \\ \hline
Random          & 15                & 98                & nan              & 72\%                  \\ 
FedCS           & 13                & 59                & nan              & 82\%                  \\ 
REFL            & 9                 & 44                & 326              & 90\%                  \\ 
RIFLES (LRU)    & 3                 & 16                & 138              & 95\%                  \\ 
RIFLES (GH)     & 3                 & 7                 & 133              & 95\%                  \\ 
\hline
\end{tabular}
\normalsize
\label{tab:toa_wisdm}
\end{table}
\begin{figure*}[!t]
    \centering
    \subfloat[\small WISDM dataset]{%
        \includegraphics[width=0.48\linewidth]{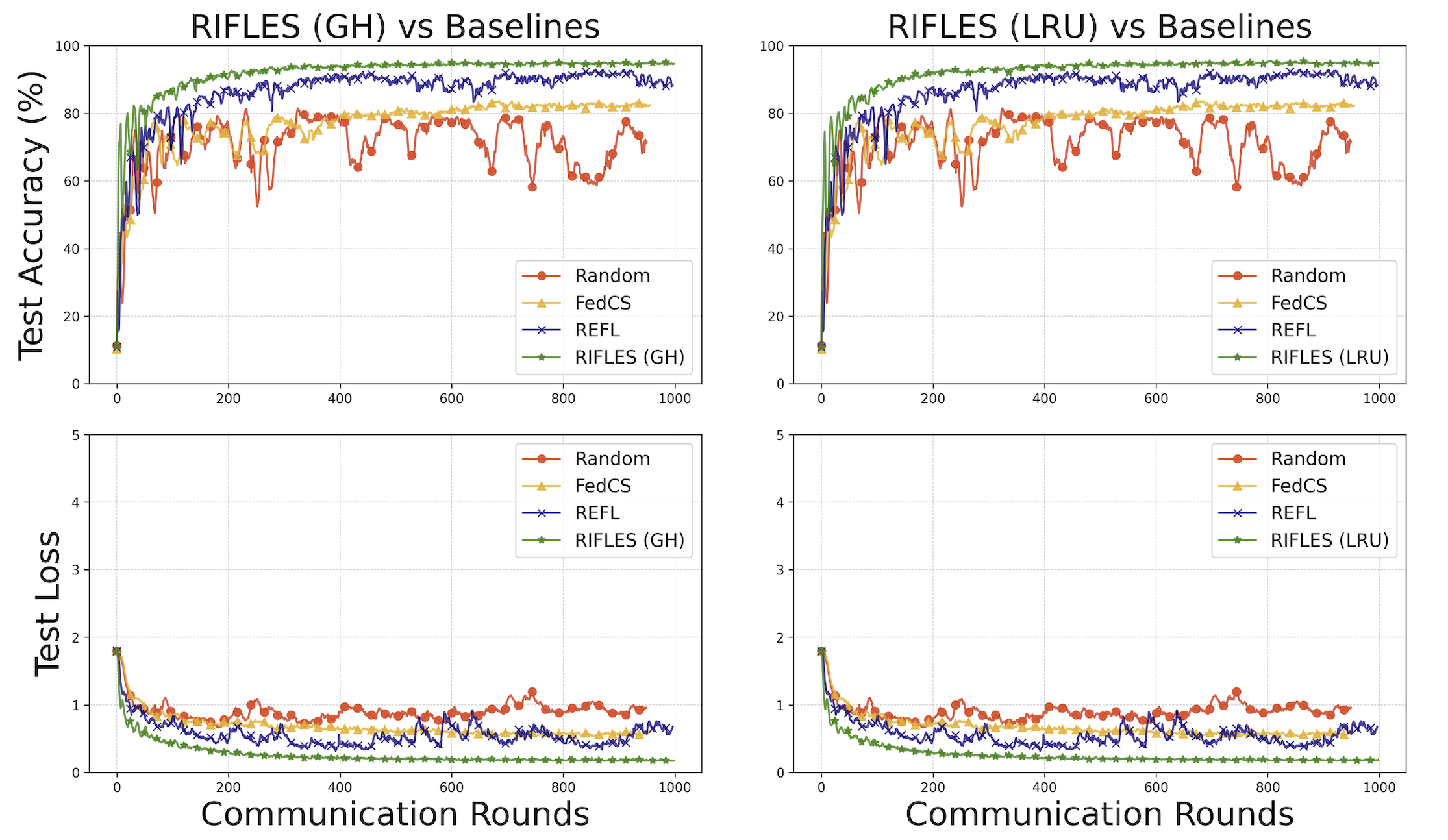}
        \label{fig:har_comparison_accuracy_loss_rifles}
    }
    \hfill
    \subfloat[\small CIFAR-10 dataset]{%
        \includegraphics[width=0.48\linewidth]{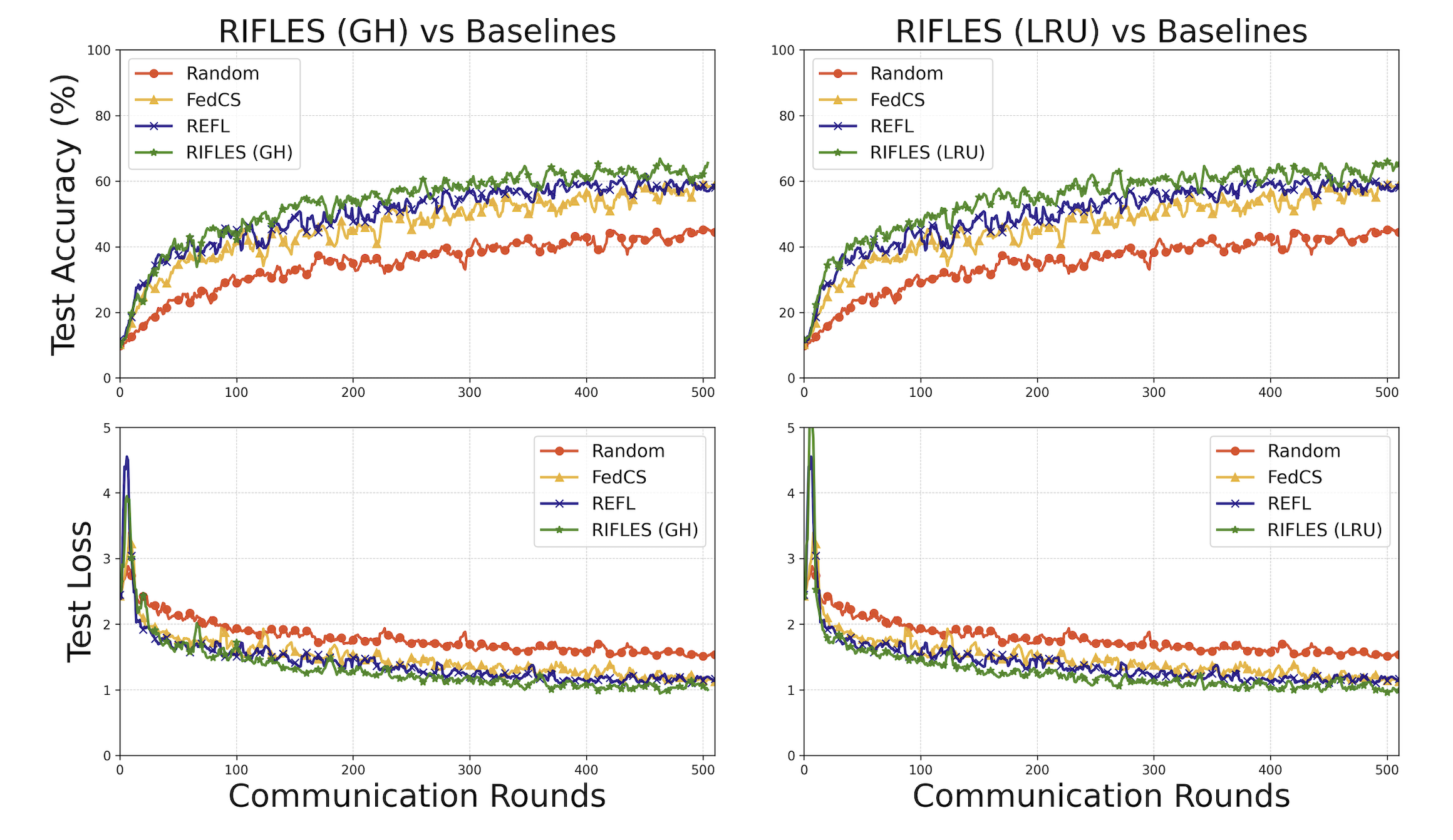}
        \label{fig:cv_comparison_accuracy_loss_rifles}
    }
    \caption{\small Comparison of test accuracy and test loss of RIFLES against baseline models (Random, FedCS, REFL).}
    \label{fig:comparison_accuracy_loss_rifles}
\end{figure*}

\subsection{Training Model Performance}
Fig.~\ref{fig:comparison_accuracy_loss_rifles} presents the test accuracy and loss over 1K rounds on WISDM and 500 rounds on CIFAR-10 for RIFLES (GH, LRU) and baseline methods, shown in the top and bottom rows, respectively.  As illustrated in Fig.~\ref{fig:har_comparison_accuracy_loss_rifles},  RIFLES (GH) achieves 94.6\% accuracy within 500 rounds, outperforming FedCS (82.5\%) and REFL (88.3\%) on WISDM, while RIFLES (LRU) further improves to 95\%; random selection lags at 71.7\%. On CIFAR-10, as shown in Fig.~\ref{fig:cv_comparison_accuracy_loss_rifles}, RIFLES (GH) reaches around 67\% and RIFLES (LRU) 69\%, compared to FedCS 59\%, REFL 60\% and random selection 47\%. In both datasets, RIFLES demonstrate faster convergence, higher stability and reduced overfitting risks, even under the relatively more non-IID settings of CIFAR-10.  Generally speaking, RIFLES in both  variants continuously maintains a 5–10\% superior accuracy and lower losses over all rounds than REFL and FedCS. Table~\ref{tab:toa_wisdm} illustrated that RIFLES achieves target accuracy levels significantly faster than REFL and FedCS, on WISDM. For example, RIFLES (GH) reaches 75\% accuracy in just 7 rounds and 90\% in 133 rounds, compared to REFL’s 44 and 326 rounds, respectively. which demonstrate how well RIFLES works to produce high-quality models with significantly fewer communication rounds.


\begin{figure}[ht]
    \centering
    \includegraphics[width=0.9\linewidth, height=0.98\linewidth]{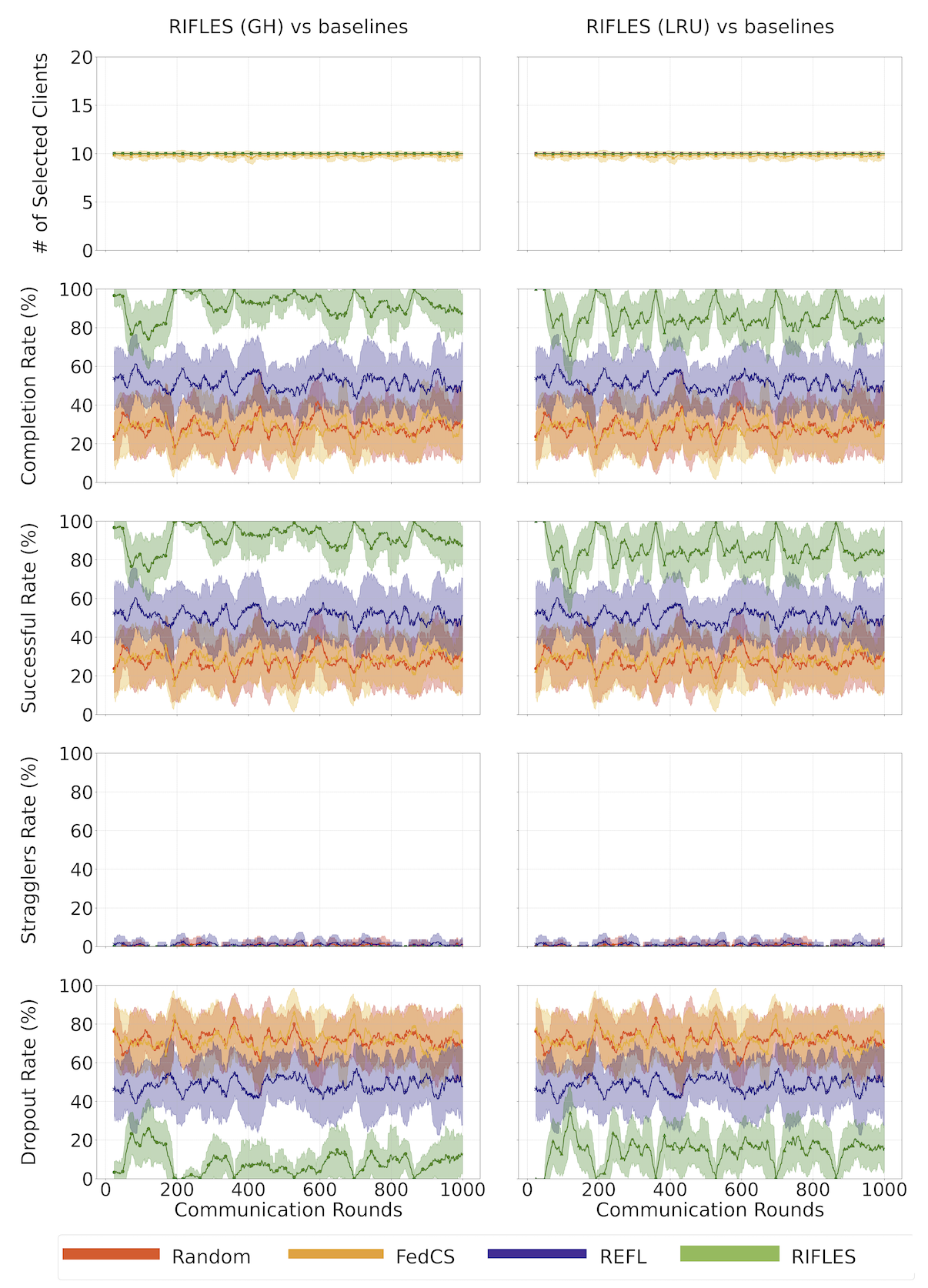}
    \caption{ Comparison of RIFLES and baselines across metrics over communication rounds.
\small Each point shows the rolling mean over 24 rounds; shaded areas represent the standard deviation.}
    \label{fig:rifles_vs_baselines}
\end{figure}

\subsection{Resource Efficiency and System Robustness}

Since our results are derived from emulation, We emulate resource usage using client counts and total time per round, tracking client outcomes and summing computation and communication times. Figure~\ref{fig:rifles_vs_baselines} illustrates the performance of RIFLES (GH) and (LRU) in comparison to baseline methods (Random, FedCS and REFL) across resource efficiency criteria.  Although RIFLES (GH) and (LRU) continuously sustain a stable count of selected clients per round (about 10), ensuring equitable participation, both variations attain superior completion rates, averaging over 85\% with negligible fluctuation, surpassing all baseline approaches.

\begin{figure*}[t]
    \centering
    \includegraphics[width=\textwidth]{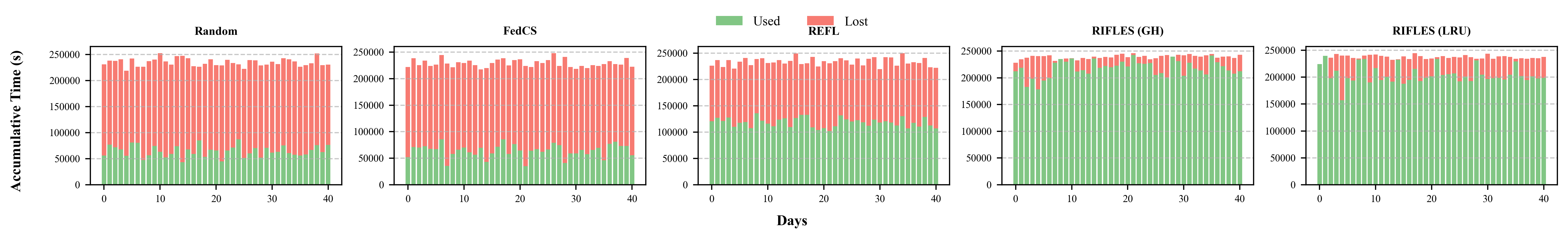}
    \caption{Daily accumulative time distribution across methods.}
    \label{fig:time_comparison}
\end{figure*}

Figure~\ref{fig:rifles_vs_baselines}, showing that both, (GH) and (LRU), maintain a stable client count per round (~10) and achieve high completion rates exceeding 85\%. Conversely, Random and FedCS demonstrate reduced and more fluctuating completion rates. This inefficiency results in the squandered computing and communication resources, as the contributions from uncompleted clients fail to enhance the global model. Moreover, RIFLES attains the lowest dropout rates of all approaches, remaining consistently below 50-60\%, while dropout rates for Random and FedCS often surpass 90\% in some rounds.

\begin{figure}
    \centering
    \includegraphics[width=0.9\linewidth]{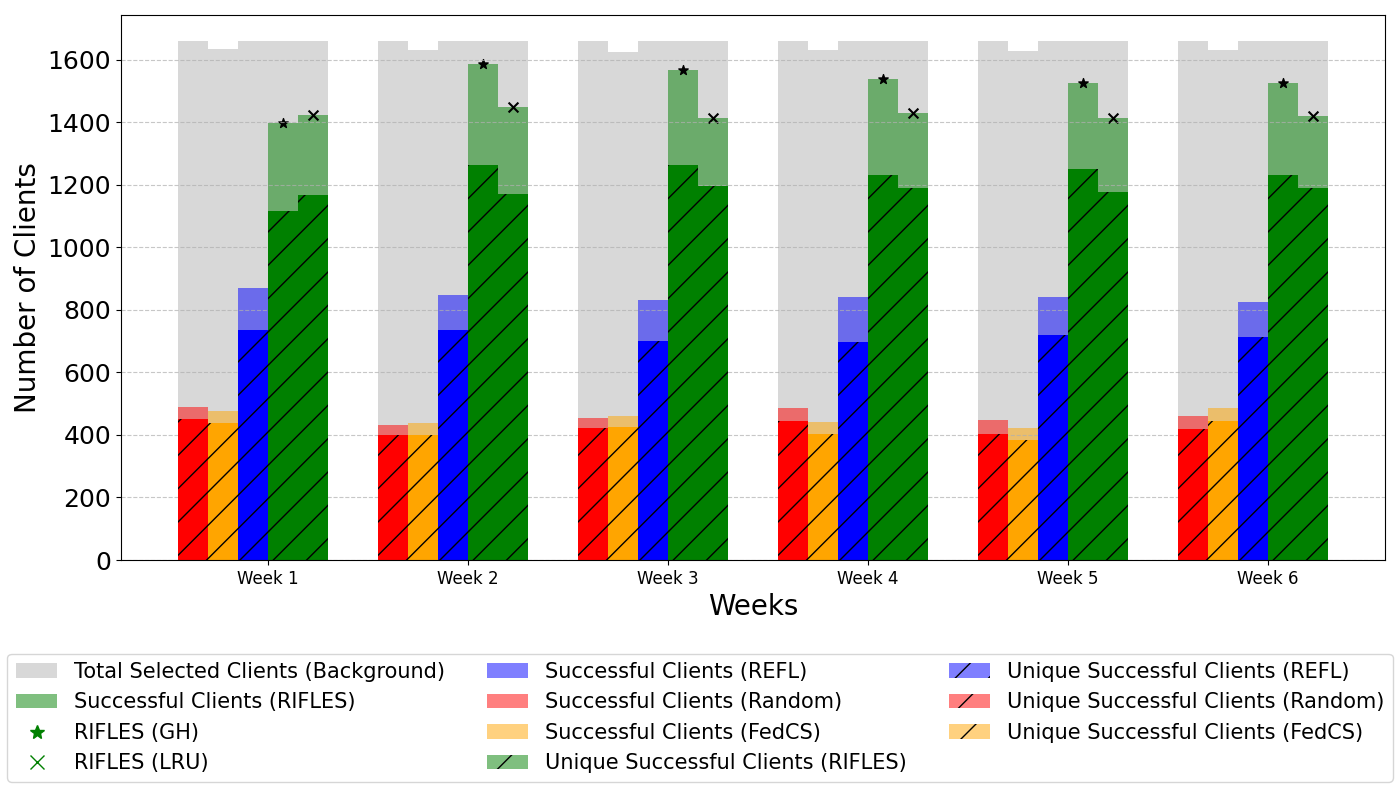}
    \caption{Weekly distribution of selected clients, successful clients and unique successful clients across methods.}

   \label{fig:rifles_vs_baselines2}
\end{figure}

Figures~\ref{fig:rifles_vs_baselines2} demonstrate RIFLES' superior client scheduling, showing higher success rates than baselines  while maintaining diverse participation by including unique clients. This strategy guarantees that clients possess sufficient time to fresher local data before re-joining and broader client involvement, helping reduce model bias in heterogeneous settings. Figure~\ref{fig:time_comparison} depicts the daily cumulative time,; while Random and FedCS suffer from high lost time, REFL reduces it significantly. However, RIFLES (GH) and RIFLES (LRU) achieve minimal lost time, highlighting their efficiency in selecting reliable clients.

\bibliographystyle{IEEEtran}

\bibliography{Bibliography.bib}

\end{document}